\documentclass[letterpaper, 10 pt, conference]{ieeeconf}  

\IEEEoverridecommandlockouts                              

\usepackage{graphics} 
\usepackage{epsfig} 
\usepackage{amsmath} 
\usepackage{tikz}
\usetikzlibrary{decorations.pathmorphing}  
\usetikzlibrary{decorations.pathreplacing} 
\usepackage{multirow}
\usepackage{subfigure}

\title{\LARGE \bf
CATCH-FORM-3D: Compliance-Aware Tactile Control and Hybrid Deformation Regulation for 3D Viscoelastic Object Manipulation
}

\author{Hongjun Ma, Weichang Li 
\thanks{*This work is supported by the National Nature Science
Foundation of China under Grant (62473158), Guangdong Basic and Applied Basic Research Foundation under Grant (2022B1515120017, 2023A1515011836).}
\thanks{The authors are all affiliated with with School of Automation Science and Engineering,
        South China University of Technology, 510641, Guangzhou, China; Institute for Super Robotics (Huangpu), 510700, Guangzhou, China. Corresponding author’s email:  {\tt\small mahongjun@scut.edu.cn}}%
}

\begin{document}

\maketitle
\thispagestyle{empty}
\pagestyle{empty}

\begin{abstract}

This paper  investigates a framework (CATCH-FORM-3D) for the precise contact force control 
and surface deformation regulation in viscoelastic material manipulation. 
A  partial differential equation (PDE) is proposed to model the spatiotemporal stress-strain dynamics, integrating
3D  Kelvin–Voigt (stiffness-damping) and Maxwell (diffusion) effects to capture the material’s viscoelastic behavior.
Key mechanical parameters (stiffness, damping, diffusion coefficients) are estimated in real time via a PDE-driven observer.
This observer fuses visual-tactile sensor data and experimentally validated forces to generate rich regressor signals.
Then, an inner-outer loop control structure is built up. In the outer loop, the reference deformation is updated by a novel admittance control law, i.e., a proportional-derivative (PD) feedback law with contact force measurements, ensuring that the system responds adaptively to external interactions.
 In the inner loop, a reaction-diffusion PDE for the deformation tracking error is formulated and then exponentially stabilized by conforming the contact surface
to analytical geometric configurations (i.e., defining Dirichlet boundary conditions). This dual-loop architecture enables the effective deformation regulation in dynamic contact environments.
 Experiments using a PaXini robotic hand 
 demonstrate sub-millimeter deformation accuracy and stable force tracking ($\pm 5 \%$ deviation). The framework advances compliant robotic interactions in applications like industrial assembly, polymer shaping, surgical treatment, and household service.

\end{abstract}

\section{INTRODUCTION}

Viscoelastic material manipulation (e.g., metal sheets, fabrics, biological tissues) presents significant robotic challenges due to their hybrid mechanical behaviors combining elastic recovery and viscous flow \cite{survey}. Autonomous handling is critical for precision tasks in industrial assembly (e.g., workpiece installation), surgical robotics (e.g., organ retraction), and domestic automation (e.g., cloth folding). Current robotic strategies fail to achieve desired deformations and stable contact forces due to insufficient processing of spatiotemporal coupling effects: stress relaxation, strain-rate damping, and geometric nonlinearity \cite{survey1}, \cite{survey2}. We propose a framework that coordinates deformation shaping with compliant interaction forces via closed-loop control, enabling efficient surface geometry transformation while maintaining low-error force profiles. This approach addresses current limitations in adaptive, safe human-centric applications by integrating viscoelastic dynamics with robotic manipulation physics.

Most viscoelastic models decouple elastic-viscous behaviors \cite{c1}. For example, the Kelvin–Voigt model captures strain-rate damping but ignores stress relaxation,  and the Maxwell model emphasizes stress diffusion while oversimplifying transient deformations. Although hybrid frameworks (e.g., Burgers models) partially unify these effects in low-dimensional regimes, they fail to resolve 3D spatiotemporal stress-strain coupling during robotic manipulation. To address this, we propose a 3D continuum model that unifies the Kelvin–Voigt model’s (stiffness-damping) and Maxwell model’s (diffusion) dynamics, explicitly governing the interplay of energy storage (stiffness), viscous dissipation (damping), and stress redistribution (diffusion). 


The parameter identification methods for viscoelastic materials often rely on offline calibration, limiting their applicability to dynamic robotic manipulation tasks \cite{pd}. For instance, frequency-domain techniques require predefined sinusoidal loading \cite{c2}, optimization-based approaches \cite{c3} depend on high-fidelity simulations but struggle with real-time adaptation, while black-box methods \cite{c4} (e.g., neural networks) face challenges in interpretability and often demand extensive training datasets. Our parameter identification scheme leverages the PDE to establish an interpretable observer structure, enabling dynamic processing of multiple regressor signals and facilitating real-time estimation of stiffness, damping, and diffusion coefficients.
The observer is operated in a data-driven way by systematically recording historical experimental force inputs and visual-tactile sensor outputs, then reutilizing diversified data to ensure sufficient information richness for resolving parameter ambiguities.


The contact force and deformation control of 3D soft objects 
 are challenging due to dynamic couplings between deformations, viscoelasticity, and contact forces. 
 Although finite element methods \cite{c5},  physics simulations \cite{c6}, machine learnings \cite{c7} are excellent in off-line modeling and design,  
their computational demands and lack of real-time adaptability make them unsuitable for dynamic  control in time-critical applications. 
On the other hand, to perform rapid response and adaptation by classical admittance control strategies \cite{c8}, 
the over-simplification by reducing dimension and localizing working region, often struggles in 3D spatial uncertain regimes.
 Thus, we propose a physics-guided admittance control framework that explicitly deals with the 
 re-plannings of reference deformation under external interaction with viscoelastic objects,
 avoiding time-consuming simulations, machine learning, or black-box optimization.
The formulated closed-loop system is a reaction-diffution PDE plant with good energy compliance and quasi-static elasticity, 
which is  stabilized by enforcing boundary deformations (essentially defining Dirichlet boundary conditions) using 
parametric geometric primitives (generated by a real-time feedback deformation information).  Such inner-outer structure ensures an exponential convergence rate of deformation fields and stabilizes contact forces under 3D viscoelastic environments.

Our contributions are summarized as follows:

(1) A unified 3D viscoelastic continuum model integrating Kelvin–Voigt and Maxwell dynamics through governing equations for energy storage, viscous dissipation, and stress redistribution in multiple-parameter fields.

(2) An observer-based parameter identification method that embeds an interpretable physical model and  regresses historical actuation forces and visual-tactile signals.

(3) A physics-guided compliance modulation for robotic manipulation that interplays deformation modulation with compliant force regulation via  viscoelastic dynamics.

(4) A boundary control strategy synthesizing Dirichlet boundary conditions through analytical geometric templates to ensure globally convergent strain fields and force stabilization under large deformations.

(5) An  inner-outer admittance control architecture achieving low-error force profiles and precise geometry transformations in dynamic tasks (e.g., industrial shaping, organ retraction), balancing computational efficiency and physical fidelity while outperforming time-consuming methods.

\section{RELATED WORK}

\subsection{Viscoelastic Materials Models}

Viscoelastic materials, widely used in soft robotics, biomedical engineering, and industrial manufacturing, exhibit complex time-dependent stress-strain behaviors that challenge precise robotic manipulation \cite{survey}. Various models \cite{survey1}, e.g., Maxwell, Kelvin–Voigt, Standard Linear Solid, have been developed to describe these behaviors, each balancing accuracy, complexity, and computational demands.  
However, key challenges remain, including real-time parameter identification, model generalizability for anisotropic and heterogeneous materials, and the integration of physics-based models with machine learning and high-fidelity sensing technologies \cite{c1}.
Current research are focused on efficient online parameter estimation, extended model applicability, and hybrid physics-data-driven approaches to enhance adaptability and robustness . In contrast, we will propose a 3D continuum model that unifies the Kelvin–Voigt model and Maxwell model, explicitly governing the interplay of energy storage, viscous dissipation, and stress redistribution.

\subsection{Viscoelastic Parameters Estimation} 

Viscoelastic parameter estimation is essential to exactly determine all parameters in a selected model
by observing the input-output behaviors of materials like polymers and biological tissues. 
 Early approaches \cite{pd} relied on quasi-static or dynamic mechanical testing (e.g., stress relaxation, creep, oscillatory shear) to fit classical models such as Maxwell, Kelvin–Voigt, or Standard Linear Solid, extracting parameters like stiffness, damping, and relaxation times through curve-fitting algorithms (e.g., least squares). Advances in inverse problem formulations enable parameter identification via finite element analysis (FEA) combined with optimization techniques (e.g., genetic algorithms, gradient-based methods), utilizing strain fields from digital image correlation or tactile sensors. Sensor fusion (visual-tactile data, force-deformation feedback) and machine learning (neural networks, Gaussian processes) further improve accuracy, particularly for nonlinear and anisotropic materials \cite{c2}-\cite{c7}.
 Challenges remain in balancing computational efficiency with model fidelity, especially under large deformations. Recent trends focus on hybrid frameworks combining physics-based models with data-driven corrections for adaptive estimation. 
So, this paper  emphasize a real-time identification using PDE model via a data-driven persistent excitation,
 to enhance precision in viscoelastic material manipulation.

\subsection{Force-Deformation Control} 

Force-deformation control is critical for robotic manipulation of viscoelastic, soft, or fragile materials, requiring simultaneous regulation of contact forces and material deformation. Traditional methods \cite{c8} like impedance and admittance control often fail to address coupled dynamics of stress-strain relationships for real-time adaptation to behaviors like creep and stress relaxation. 
Although sensor fusion (visual, tactile data) enhances deformation estimation, with high-resolution tactile arrays (e.g., GelSight, BioTac) providing micron-scale feedback, it is difficult for classical control methods (MPC, PID) to 
adjust reference trajectories based on force-deformation objectives while stabilizing tracking errors. Finite element methods \cite{c5}, physics simulations \cite{c6}, and machine learning \cite{c7}  improves performance in offline modeling and design, but their high computational demands and inability to adapt in real time render them impractical for dynamic control in time-sensitive applications. Compared
with existing results  \cite{c8}-\cite{c12}, we explore an admittance controller for contact force and surface deformation regulation, and 
stabilize the tracking for reference deformation with a reaction-diffusion PDE-based advanced control.

\section{Unified 3D Viscoelastic Continuum Model}

A straightforward way to model the interaction between robotic manipulators and target object is by establishing a force-displacement relationship. This is often represented by analytical models combining springs and dampers (Fig. \ref{fig:model}).

\begin{figure}[htbp]
    \centering
    \includegraphics[height=18mm, width=60mm]{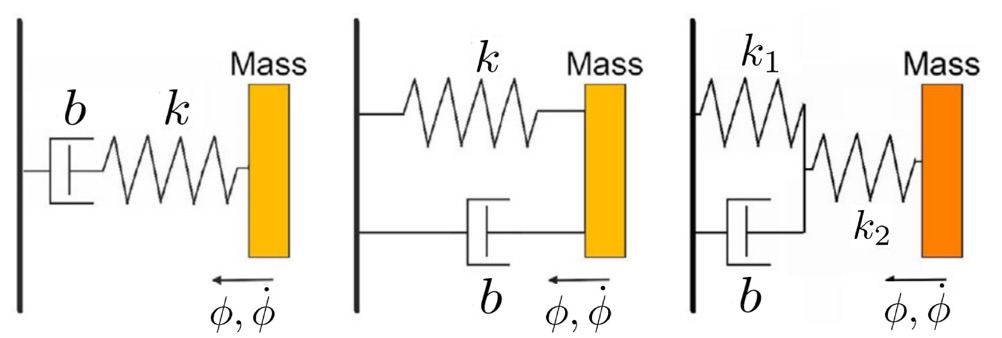}   
    \caption{Kelvin-Voigt, Maxwell and Burgers Models in 1D.}
    \label{fig:model}
\end{figure}

\subsection{Kelvin–Voigt and Maxwell Effects in 1D}

The Kelvin–Voigt model consists of a spring in parallel with
a damper  and it is dynamic equation is described by
\begin{align}
  f(t) = k \phi(t) + b \dot{\phi}(t)
  \label{eq:kelvin_voigt}
\end{align}
which includes the object's reaction force $f$ under strain; the indentation depth $\phi$, measured as object displacement from its rest position; the deformation velocity $ \dot{\phi}$; and the elastic and damping coefficients, $ k $ and $ b $, respectively. This equation is
suitable for modeling materials that exhibit creep and delayed elasticity, such as soft tissues or polymers.

The Maxwell model is represented by the series of a spring
and a damper and is expressed as the following equation
\begin{align}
  f(t) = b \dot{\phi}(t)-\alpha \dot{f}(t)
  \label{eq:maxwell}
\end{align}
where $\dot{f}$ is the derivative of the exerted force and $\alpha= b/k$, which is
useful for modeling materials that exhibit stress relaxation, such as biological fluids or viscoelastic liquids.

The Kelvin-Voigt model and the Maxwell model are generalized into the  Burgers force-displacement relationships. 
\begin{align}
  f(t) = k \phi(t) + \beta \dot{\phi}(t)-\gamma   \dot{f}(t)
  \label{eq:general1D}
\end{align}
where  $k = k_1k_2/(k_1 + k_2)$, $\beta=bk_2/(k_1 + k_2)$, $\gamma=b/(k_1 + k_2)$, $k_1$, $k_2$
and $b$ are the elastic and damping coefficients.

\subsection{3D Viscoelastic Continuum Model}

\begin{figure}[htbp]
    \centering
    \includegraphics[height=20mm, width=50mm]{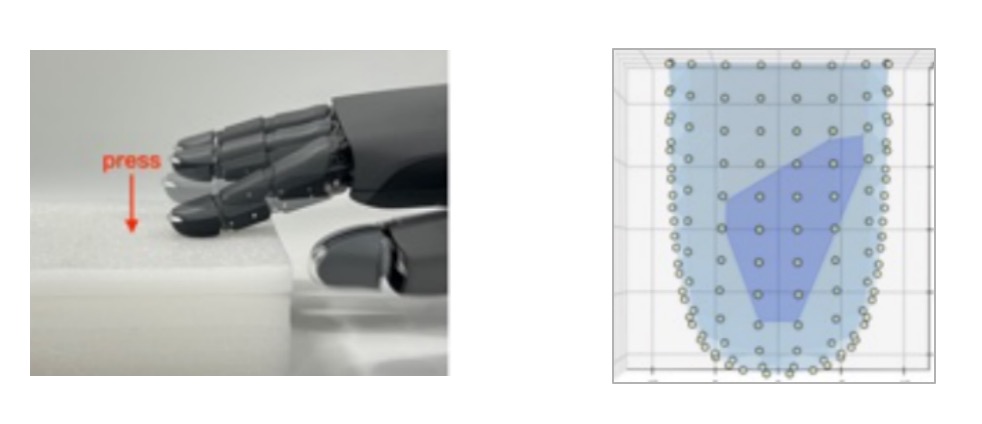}   
    \caption{The visual-tactile sensor for contact force-deformation.}
    \label{fig:sensor}
\end{figure}

\begin{figure}[htbp]
    \centering
    \includegraphics[height=40mm, width=50mm]{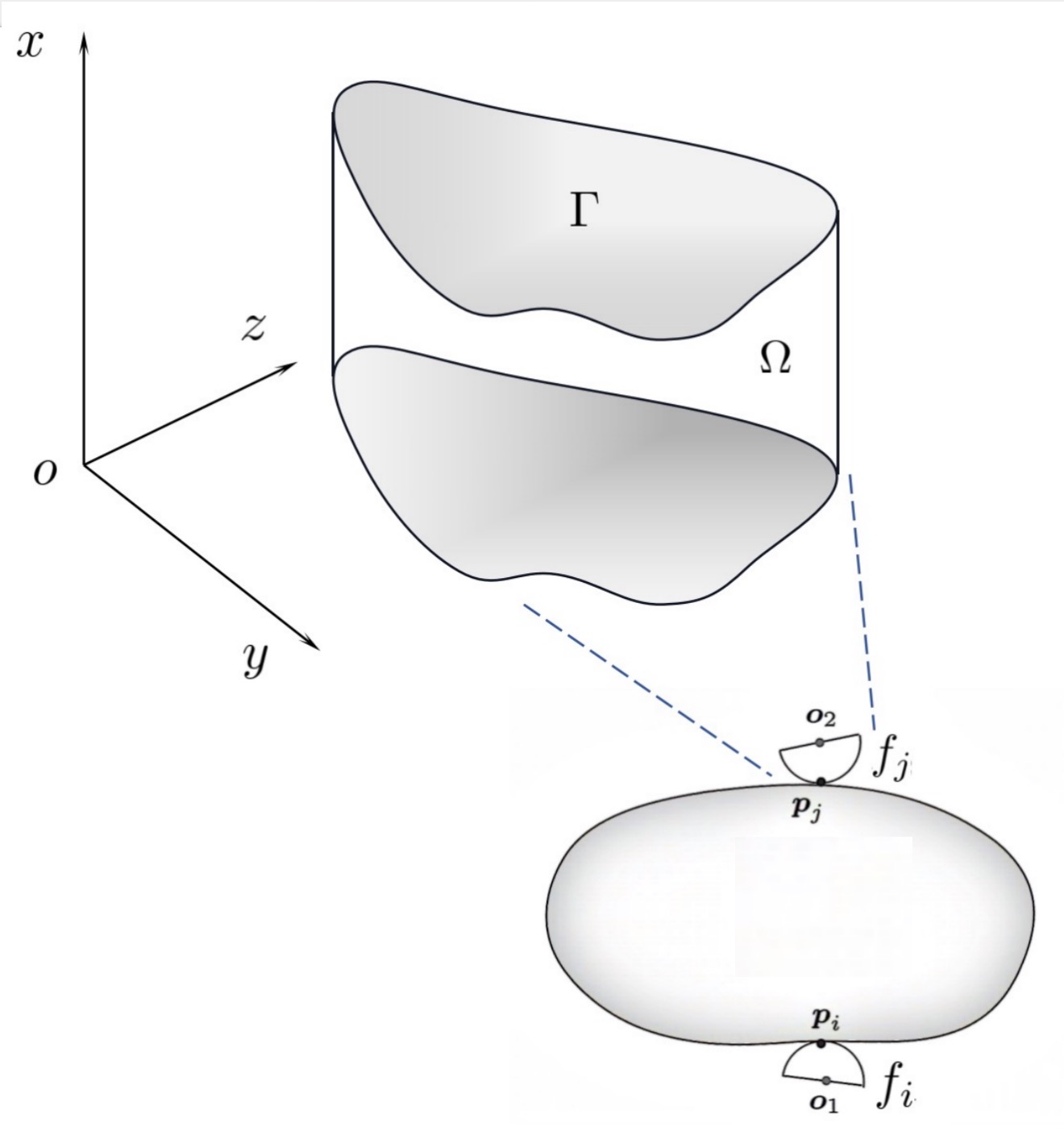}   
    \caption{The domain $\Omega$ for the manipulation.}
    \label{fig:geometry}
\end{figure}

To further address the limitations of 1D Kelvin–Voigt and Maxwell models, we propose  a PDE which unifies these dynamics in a 3D spatiotemporal continuum as below
\begin{align}
&\dot{\phi}(t, x,y,z)=\epsilon \Delta \phi (t, x,y,z)+ a_1 f (t, x,y,z)\notag\\
&+a_2 \dot{f} (t, x,y,z)+ \lambda \phi(t, x,y,z)
 \label{eq:3dpde}
\end{align}
where $\phi(t, x,y,z)$ is the deformation scalar field at time $t$ and spatial position $(x,y,z)$, representing the displacement of the object relative to its reference configuration (undeformed state);  $f(t,x,y,z)$ is a scalar filed describing the effect of
the contact force projecting on the direction of the gradient $\nabla\phi$
representing the rate and direction of change of the deformation; 
$\Delta\phi$ is the Laplace operator, representing the spatial variation, diffusion, and resistance to deformation. This formulation generalizes the 1D  viscoelastic model  into a 3D counterpart but without directly deal with a vector field of contact force, useful to give an analytical design procedure.

Here, we don't apply directly the 3D vector fields to describe contact forces and deformations while
using simplified scalar fields, assuming that the shear force is zero which can be compensated in manipulation if it's not.
As shown in Fig. \ref{fig:sensor}, the contact force-deformation scalar fields are computed
by the sensing data from the applied visual-tactile sensor.

Without loss of generality, consider the geometry (Fig. \ref{fig:geometry}) for $(x,y,z)\in\Omega$, where $\Omega$ is a cylinder with top and bottom of arbitrary shape $\Gamma$. This configuration of the domain $\Omega$ is convenient because we can view the problem as many 1D problem with $0\leq x\leq \delta$ by fixing $y, z$. 
We assume the following Dirichlet boundary conditions on the boundary $\partial \Omega$ 
\begin{align*}
\phi(x, y,z)=0, ~ (x,y,z)\in \partial \Omega \setminus \{x=\delta\}
\end{align*}
the actuation force is applied at the top of the cylinder $x=\delta$,
where the actuation force is applied as  in the later sections.
The deployment of visual and tactile data for contact force and
deformation acquisition is performed by  visual-tactile sensors, 
equipped on robotic hand and outputs data matrix of 3D surface position change
and 3D contact force.  Note that such formulation is also a reasonable approximation 
if the contact force is actuated on a flat enough area or the contact area is relatively small compared 
to a large object.

Obviously, $\epsilon>0$, $a_1=\beta^{-1}$,
$a_2=\beta^{-1}\gamma$, $\lambda=-\beta^{-1}k$  are the parameters describing
mechanical properties of energy storage (stiffness), viscous dissipation (damping), and stress redistribution (diffusion).
This integration bridges classical distributed parameter models with the needs of robotic manipulation, enabling adaptive handlings of viscoelastic materials in dynamic tasks. We will omit the variable $(t,x,y,z)$ in the formulas if it does not cause confusion.
\subsection{Problem Statement}

Now, the considered problem is claimed, that is,  the parameter identification is first studied to 
achieve precise estimations for all viscoelastic parameters in the PDE, then, based on the established 3D 
viscoelastic continuum model,  the contact force and deformation control will be explored to
realize the trackings of the scalar fields of $f(x,y,z)$ and $\phi(x,y,z)$ to their reference fields $f_d(x,y,z)$, 
$\phi_d(x,y,z)$ respectively. To this end, an admittance control architecture is developed such that 
the reference deformation $\phi_d(x,y,z)$ is on-line modified to support  compliant contact force 
interactions, which is rapidly tracked by conforming the Dirichlet boundary conditions to 
analytical geometric configurations.

\section{Observer-Based Parameter Identification}

\subsection{Adaptive Observer for Dynamic PDE System} 

Our parameter identification scheme (Fig. \ref{fig:observer}) is based on the proposed  PDE (\ref{fig:observer}) which
is a dynamic system with contact forces inputs and visual-tactile sensing outputs. Since
the PDE has an interpretable structure, we can use it directly to design an adaptive observer, enabling dynamic processing of multiple regressor signals and facilitating real-time estimation of stiffness, damping, and diffusion coefficients.
This observer employs a copy of the PDE plant plus with an additional feedback term that is related with a regress signal
\begin{align*}
\dot{\hat\phi}=\hat\epsilon \Delta \phi + \hat a_1 f +\hat a_2 \dot{f}+ \hat\lambda \phi+L(\phi-\hat\phi)
\end{align*}
in which $\hat\phi$, $\hat\epsilon$, $\hat a_1$, $\hat a_2$, $\hat\lambda$ are the estimations for the counterpart in the PDE,  $L= L'+ K \psi^T\psi$, $L'>0$, $\psi\in \mathbf{R}^{4\times 1}$ and $K>0$ 
\begin{align}
\dot{\hat\epsilon}=-K \psi_1 (\phi -\hat \phi), ~\dot{\hat a}_1=-K \psi_2 (\phi -\hat \phi),  \notag\\
\dot{\hat a}_2=-K \psi_3 (\phi -\hat \phi), ~\dot{\hat\lambda}=-K \psi_4 (\phi -\hat \phi)
\end{align}
where $\psi=[\psi_1, \psi_2, \psi_3, \psi_4]^T$ is a regressor signal which is  generated by a  system filtering the signal
$\Psi=[\Delta \phi,  f, \dot{f},  \phi]^T$
\begin{align}
\dot{\psi}=-L' \psi+\Psi.
\label{eq:regressor}
\end{align}

\begin{figure}[htbp]
    \centering
    \includegraphics[height=40mm, width=70mm]{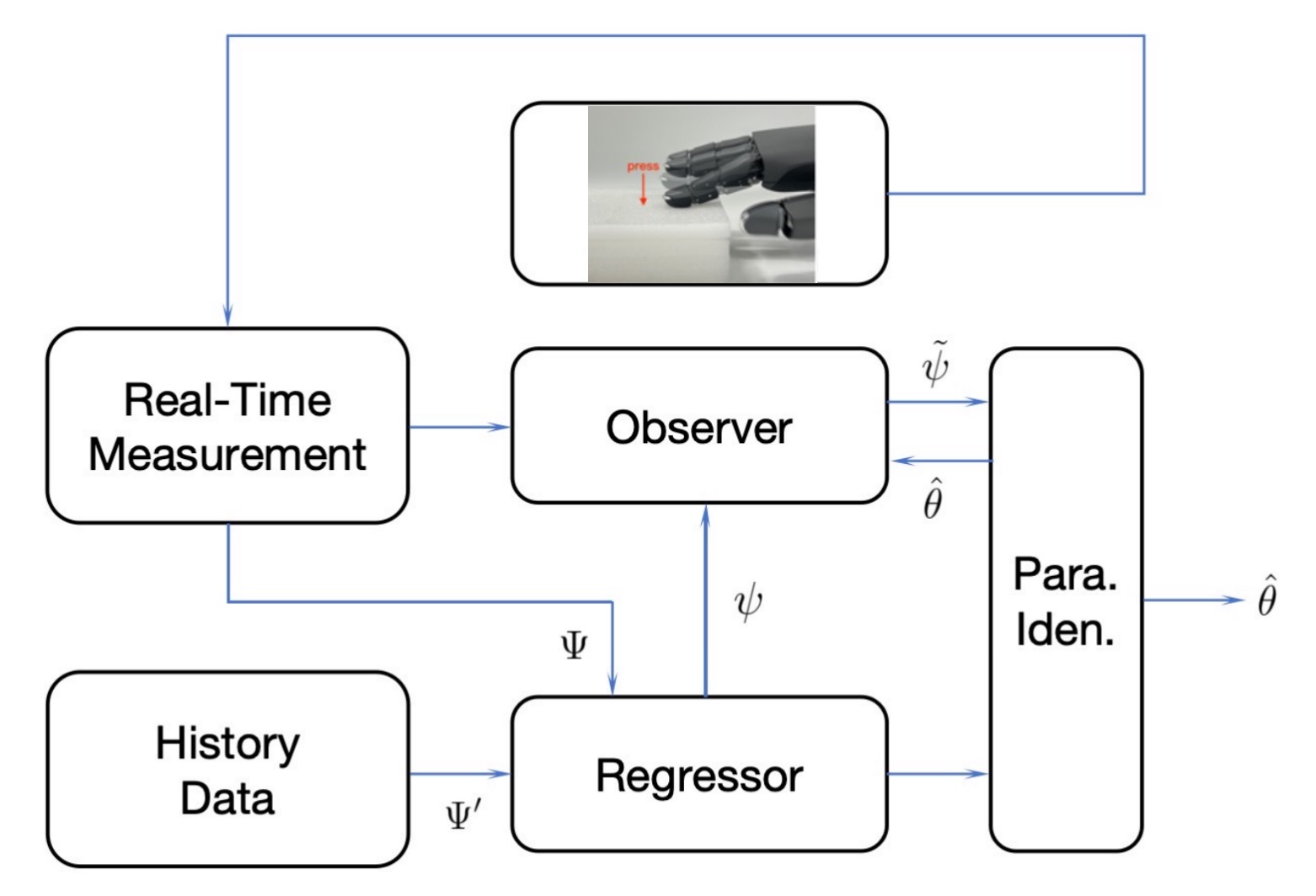}   
    \caption{PDE-based observer \& parameter estimator.}
    \label{fig:observer}
\end{figure}

\subsection{Historical Data Replay for Persistent Excitation} 

To ensure the convergence of estimation of $\hat\theta=[\hat\epsilon, \hat{a}_1, \hat{a}_2, \hat{\lambda}]^T$ to its exact value,
a practical strategy is  to inject experimentally validated input signals containing enough frequency content to excite all system modes. This condition is called as `persistent excitation condition' that also provides robustness to noise and prevent ill-conditioning in regression.
\begin{align}
\int_{t}^{t+\tau}\psi \psi^T dt \succ 0, ~\tau >0.
\label{eq:PE1}
\end{align}
To this end, the regressor signal $\psi$ is re-generated by combing current and historical dada of the actuation force and visual-tactile sensing $\Psi=[\Psi_1 \cdots, \Psi_m]$, 
$\phi=[\phi_1 \cdots, \phi_m]^T$.
\begin{align}
\dot{\psi}=-L' \psi+\Psi, ~L'\succ 0,~L'\in \mathbf{R}^{m\times m}, ~\Psi\in \mathbf{R}^{4\times m}
\label{eq:regressor1}
\end{align}
which  generates $L=L'+K\psi^T \psi \succ 0, K\succ 0$  to be used by
\begin{align}\label{Psi}
\dot{\hat\phi}=\Psi^T\hat\theta+L(\phi-\hat\phi), ~\phi \in \mathbf{R}^{m\times 1}, ~\hat\theta \in \mathbf{R}^{4\times 1}.
\end{align}
The estimation of $\theta$ is written in the following compact form
\begin{align}\label{theta}
\dot{\hat\theta}=K \psi(\phi-\hat\phi), ~\phi \in \mathbf{R}^{m\times 1}, ~\hat\theta \in \mathbf{R}^{m\times 1}.
\end{align}
This means that we can diversify actuation force inputs and collect the corresponding visual-tactile sensing data to ensure persistent excitation, enabling accurate estimations for all elements in $\theta$ until each one is precisely identified, i.e., the
autocorrelation matrix $\psi \psi^T$ is full rank over a time interval.
\newtheorem{lemma}{\bf Lemma}
\begin{lemma}\label{lemma1}
Under the PE condition (\ref{eq:PE1}), the parameter estimation  (\ref{theta}) with regressor (\ref{eq:regressor1}) and observer (\ref{Psi})  satisfies  
\begin{align*}
\lim_{t\rightarrow\infty}\hat\theta=\theta, ~\lim_{t\rightarrow\infty}\hat\phi=\phi.
\end{align*}
\end{lemma} 
\begin{proof} (sketch) Define the estimation error $\tilde\theta=\theta-\hat\theta$, $\tilde\phi=\phi-\hat{\phi}$, define an
auxiliary variable $\eta=\tilde \phi-\psi^T \theta$, then
\begin{align*}
\dot{\eta}= \Psi\tilde\theta-L(\eta+\psi^T\tilde \theta)-\dot{\phi}\tilde\phi-\phi\dot{\tilde\theta}.
\end{align*}
By inserting the estimation laws $\tilde\theta$ and the dynamics of $\psi$,
\begin{align*}
\dot{\eta}=-L'\eta, ~\dot{\tilde\theta}=-K\psi\psi^T\tilde\theta-K\psi^T\eta
\end{align*}
 implies the exponential convergences of $\eta$, $\tilde\theta$. \end{proof}

\section{Physics-Guided Deformation Planning}

After the parameters $\epsilon$, $a_1$, $a_2$, and $\lambda$ are exactly identified, the ideal contact force can be designed using the PDE that explicitly connects contact forces with deformation profiles. This approach eliminates the need for time-consuming simulations, machine learning techniques, or black-box optimization methods, allowing for an efficient and precise determination of the desired contact force. A dual-loop architecture in Fig.\ref{fig:framework} is presented, which ensures a safe and effective force-deformation interaction,  making it suitable for robotic manipulation and human-robot collaboration when subject to 3D viscoelastic objects.
Recall the proposed PDE
\begin{align}
\dot{\phi}=\epsilon \Delta \phi + a_1 f +a_2 \dot{f}+ \lambda \phi,
 \label{eq:3dpdeid}
\end{align}
we can define the reference fields $f_d$ and $\phi_d$, then it implies
\begin{align}
\dot{\phi}_e=\epsilon \Delta \phi_e + a_1 f_e +a_2 \dot{f}_e+ \lambda \phi_e,
 \label{eq:3dpdeerror}
\end{align}
where $f_e$ and $\phi_e$ are the tracking errors of two scalar fields. 
Design an admittance control law of   $\phi_e$ according to $f_e$, $\dot f_e$
\begin{align}
&a_1 f_e +a_2 \dot{f}_e= \lambda_1 \phi_e+\lambda_2 \dot{\phi}_e, 
\label{eq:f2d}
\end{align}
To design a robust  admittance control loop,  we analyze its 
passivity and stability  that is to be used in the outer-loop.

The passivity requires that the system does not generate energy, meaning the integral of the input energy must be non-negative. For time-invariant systems, passivity is equivalent to the positive realness of the transfer function.
Taking the Laplace transform (assume zero initial conditions), then
\begin{align*}
G(s) = \frac{\Phi(s)}{F(s)} = \frac{a_1 + a_2 s}{\lambda_1 + \lambda_2 s}.
\end{align*}
For $ G(s) $ to be positive real, the denominator polynomial $ \lambda_1 + \lambda_2 s $ have roots with negative real parts, i.e.,   $\lambda_1/\lambda_2 > 0$.
For all $ \omega \in \mathbf{R} $, the real part of $ G(j\omega) $  is positive means that
\[
\text{Re}\{G(j\omega)\} = \frac{a_1 \lambda_1 + a_2 \lambda_2 \omega^2}{\lambda_1^2 + \lambda_2^2 \omega^2}>0,
\]
 so if $ a_1 \lambda_1 > 0 $ and $ a_2 \lambda_2 > 0 $, then  the system is passive.

\begin{lemma}\label{lemma2} If the parameters $\lambda_1$ and $\lambda_2$ satisfy that 
\begin{align*} 
    \lambda_1 > 0, ~ \lambda_2 > 0, 
\end{align*}
the admittance control loop system are passive and stable.
\end{lemma} 
\begin{proof}  It can be proved as described above.\end{proof}

The selections of $\lambda_1$ and $\lambda_2$ are of a clear physical interpretation, i.e.,
 $ a_1, a_2 $ control the admittance's response to static and dynamic force changes, and
 $ \lambda_1, \lambda_2 $ determine the stiffness and damping characteristics of the deformation adjustment,
 and the diffusion characteristic of the deformation tracking (detailed explanation in the next subsection).
 Here, $a_2$ is specially used to consider the sensitivity to the derivative of $f_e$ this is different to 
 the conventional admittance control strategy, due to the Maxwell effect.
 
If we consider for a linear inner-loop system, then the closed-inner-outer-loop system can be stabilized, by designing
the following closed-loop transfer function to be stable
\[
T(s) = \frac{G(s) C(s)}{1 + G(s) C(s)}
\]
where $ C(s) $ is the inner-loop controller, the stability can be verified  by ensuring all poles lie in the left-half plane.

However, the inner-loop system is essentially a complex PDE plant of 3D surface deformation,  which
is to be designed with an advanced controller in the next section to achieve a rapid tracking of a dynamic reference deformation.

\section{Boundary Control by Geometric Templates}

\begin{figure}[htbp]
    \centering
    \includegraphics[height=35mm, width=70mm]{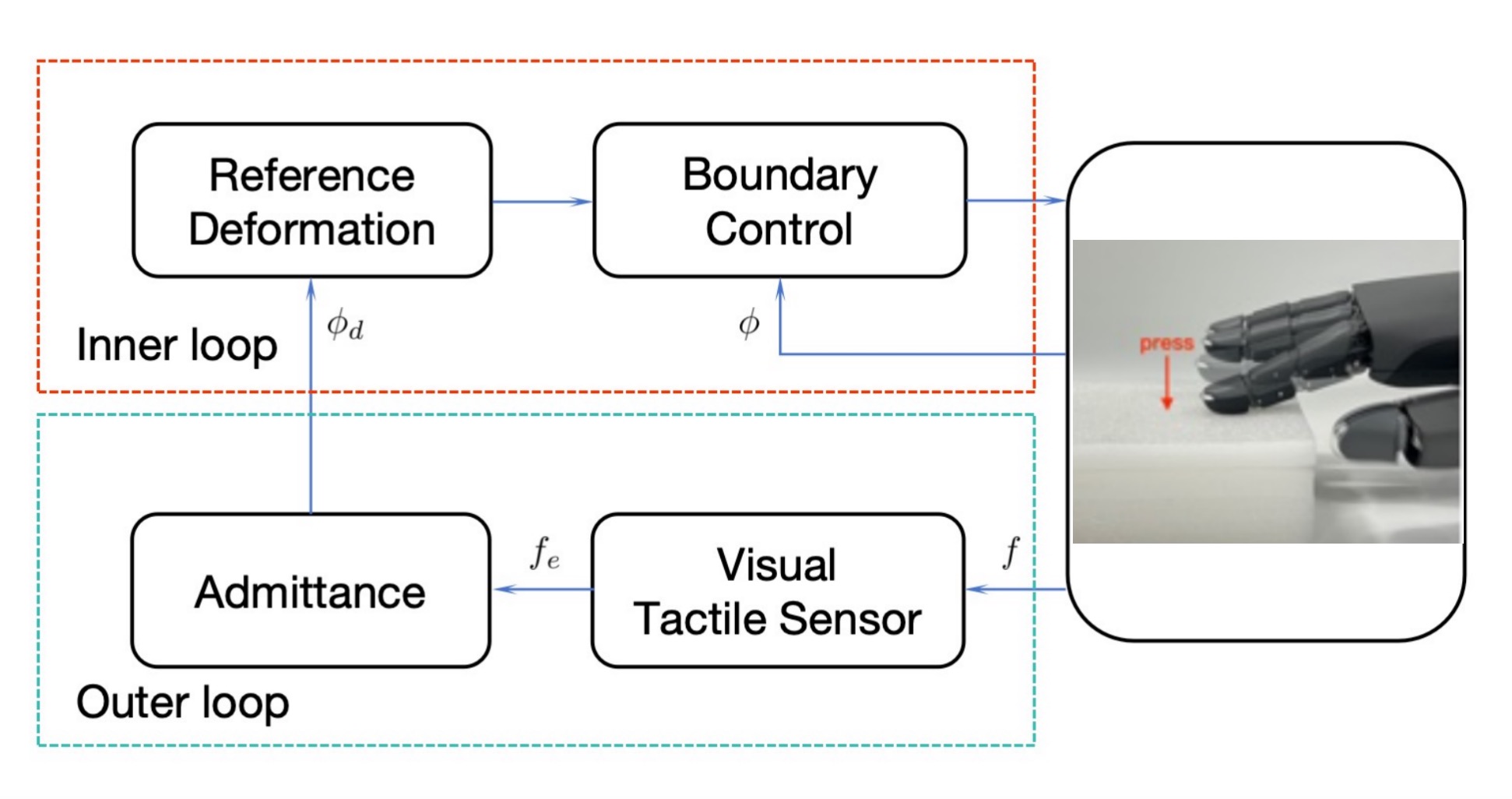}   
    \caption{The dual-loop control architecture of contact force-deformation.}
    \label{fig:framework}
\end{figure}

Now, we can discuss the inner-loop system,  by applying the admittance control law (\ref{eq:f2d}), the 
PDE of tracking errors $f_e$ and $\phi_e$ are formulated into the following new form, i.e.,
\begin{align}
(1-\lambda_2)\dot{\phi}_e=\epsilon \Delta \phi_e + (\lambda+\lambda_1) \phi_e,
 \label{eq:3dpdeerror1}
\end{align}
Choose the parameters satisfying $0<\lambda_2<1$, $\lambda_1>0$, and denote $\epsilon^*=(1-\lambda_2)^{-1}\epsilon$, $\lambda^*=(1-\lambda_2)^{-1}(\lambda+\lambda_1)$, then
\begin{align}
\dot{\phi}_e=\epsilon^* \Delta \phi_e + \lambda^* \phi_e
 \label{eq:3dpdeerror2}
\end{align}
is the resultant system which is a reaction-diffusion PDE. 
The solution of this PDE is well-behaved \cite{c10}, that is, the energy of the solution decays over time, and the maximum and minimum values of the solution typically occur at the initial or boundary conditions, and no new extrema arise in the interior region, the deformation distribution tends to become more uniform over time. Our goal is to drive deformation tracking error to zero by  a control mechanism.

In general, we can design the solution of PDE (\ref{eq:3dpdeerror2}), by considering the initial deformation as  $\phi_d(0, x,y,z)=0$  
and the on-line updated desired deformation $\phi_d(\infty, x,y,z)$,  the solution to the equation \(\dot{\phi}_e = \epsilon^* \Delta \phi_e + \lambda^* \phi_e\) can be constructed through the method of separation of variables and eigenfunction expansion of the Laplace operator. If the geometric shape is a cube with side length \(L\) and \(\phi_e = 0\) on the boundary, then the solution is written in the form below:
\begin{align*}
&\phi_e(t, x, y, z) = \sum_{n, m, p} C_{nmp} e^{-\epsilon^* \left( \left( \frac{n \pi}{L} \right)^2 + \left( \frac{m \pi}{L} \right)^2 + \left( \frac{p \pi}{L} \right)^2 + \frac{\lambda^*}{\epsilon^*} \right) t}\\
&\times \sin \left( \frac{n \pi x}{L} \right) \sin \left( \frac{m \pi y}{L} \right) \sin \left( \frac{p \pi z}{L} \right), ~(x,y,z)\in \Omega
\end{align*}
where \(n_x, n_y, n_z \in \mathbf{N}\), and the coefficients \(C_{n_x n_y n_z}\) are determined by expanding the initial condition into a three-dimensional Fourier series. In general geometric settings, the solution to the equation \(\dot{\phi}_e = \epsilon^* \Delta \phi_e + \lambda^* \phi_e\) satisfies that
\begin{align*}
\phi_e(t, x,y,z) = \sum_{n} C_n e^{(\lambda^* - \epsilon^* \mu_n) t} H_n(x,y,z)
\end{align*}
where \(\mu_n\) and \(H_n(x,y,z)\) are the eigenvalues and eigenfunctions of the Helmholtz equation on the geometric domain \(\Omega\); \(C_n\) is determined by projecting the initial condition onto the eigenfunction space;
the temporal evolution of the solution is governed by the exponential factor \(e^{(\lambda^* - \epsilon^* \mu_n)t}\), with the specific behavior depending on the relative magnitudes of \(\lambda^*\) and \(\epsilon^* \mu_n\).
For specific geometric shapes (such as spheres, cylinders, cubes, etc.), the eigenfunctions \(H_n(\mathbf{x})\) and eigenvalues \(\mu_n\) can be further expressed analytically in terms of special functions (such as Bessel functions, spherical harmonics, Fourier basis functions). If the geometric shape is complex or lacks an analytical solution, numerical methods (such as FEM) are required to solve the eigenvalue problem.

Apparently, a numerical solution is very inconvenient or even impossible
 in practical application, so our idea is to find a
solution from the point of view of control theory. A feedback mechanism is introduced 
in the inner-loop system, where the Dirichlet boundary condition is used as a controller to
support a rapid convergence to the reference deformation. Motivated by the boundary control for basic stabilization in \cite{c10}, here the Dirichlet boundary condition plays
a key role to drive the deformation to its reference, that is, we specially use the
deformation tracking error to implement a feedback.

Consider the Dirichlet boundary condition below on  $\partial \Omega$ 
\begin{align*}
&\phi_e(x, y,z)=0, ~ (x,y,z)\in \partial \Omega \setminus \{x=\delta\},\\
&\phi_e(1, y,z)=U_e(t, y,z), ~(y,z)\in \Gamma
\end{align*}
where $U_e(t, y,z)$ is the regulation boundary control law 
\begin{align*}
\phi_e(\delta,y,z)=-\frac{\lambda^*}{\epsilon^*}\int^\delta_0 \xi \frac{I_1\left(\sqrt{\frac{\lambda^*}{\epsilon^*}(\delta^2-\xi^2)}\right)} {\sqrt{\frac{\lambda^*}{\epsilon^*}(\delta^2-\xi^2)}}\phi_e(\xi, y, z) d\xi
\end{align*}
where $I_1(x)=\sum^\infty_{m=0} \frac{(x/2)^{2m+1}}{m!(m+1)!}$ is a 1st order revised Bessel function,
 $\phi_e(\xi, y, z) $ is used in a weighted integral on $[0,\delta]$.
 Then, the reference position of the end-effector is computed by the weighted summed distance of all deformation points.

\begin{lemma}\label{lemma3}  By using the control $U_e(t, y,z)$,
the inner-loop deformation tracking error  exponentially converges to zero. 
\end{lemma} 
\begin{proof}  (sketch)
Introduce the nonlinear transformation
\begin{align*}
&w(x,y,z) = \phi_e(x,y,z) - \int_{0}^{x} k(x,\xi)\phi_e(\xi,y,z)  d\xi\\
&k(x,\xi)=-\frac{\lambda^*}{\epsilon^*} \xi \frac{I_1\left(\sqrt{\frac{\lambda^*}{\epsilon^*}(x^2-\xi^2)}\right)} {\sqrt{\frac{\lambda^*}{\epsilon^*}(x^2-\xi^2)}}
\end{align*}
where $k(x,y)$ is a kernel function satisfying the conditions 
\begin{align*}
 &k_{xx}(x, y) - k_{yy}(x, y) = \lambda^* k(x, y),\\
 & k(x, 0) = 0, ~k(x, x) = -\frac{\lambda^* x}{2}.
\end{align*}
By calculations, the following target system of $w(x,y,z)$ 
\begin{align*}
\dot w =\Delta w, ~w(0,y,z)=w(\delta,y,z)=0
\end{align*}
 is exponentially stable, and thus the tracking error $\phi_e$.
\end{proof}

\newtheorem{therem}{\bf Therem}
\begin{therem}\label{therem1}
By implementing the admittance control law (\ref{eq:f2d}) in outer-loop and the boundary control  $U_e(t, y,z)$ in the inner-loop,
the closed-loop system is stabilized by performing  the contact force and deformation tracking control.
\end{therem} 
\begin{proof} It is a direct conclusion of lemmas 2 and 3. \end{proof}

We emphasize that the proposed control strategy is also applicable for more general geometric settings,
where the contact force is actuated on a flat enough area or the contact area is relatively small compared 
to a large object. What's more, since the form of reaction-diffusion PDE is linear in the force-deformation fields,
the solution under multiple Dirichlet boundary conditions,  can be analyzed using the principle of superposition,
that is, the global solution is the sum of all local ones.  So our method is extendable to 
a large viscoelastic object operated by multiple robotic end effectors.

\section{EXPERIMENTAL VALIDATION}

The proposed CATCH-FORM-3D framework is validated by examing its performances across diverse materials and conditions, including (1)  force control precision and (2) deformation tracking accuracy.  Typical
material properties of stiffness, viscoelasticity, and surface geometry, are specially considered as independent variables, with force tracking error (FTE, $N$) and composite deformation error (CDE, $mm^2$) serving as the primary dependent metrics.
This experiment employs a PaXini hand (DexH13)  mounted on  a Realman RM arm, which is equipped with
high-precision tactile arrays (PX6AX-GEN2-DP-M2826, spatial resolution: 2.0–2.5 $mm$, force resolution: 0.01 $N$)  
that is able to acquire data at 60$Hz$, with statistical significance assessed via the repeated-measures ANOVA ($\alpha=0.05$, $n=5$ trials/condition).



\subsection{Material Properties and Setup}

To evaluate the controller's performance on diversified viscoelastic objects, we selected a set of calibrated material samples,
with typical elastic, viscoelastic, and rigid properties respectively (Tab.\ref{tab:material_properties}). 
A variety of industrial, medical or household objects are selected (Fig.\ref{fig:obj_press}), ranging from silicone blocks, metal components, foam to fragile items. Material property testing was implemented in accordance with the standards ASTM D2240 (Shore hardness), ASTM E8 (tensile properties), and ASTM D695 (compression testing), using calibrated measurement equipment (Instron 5569A, MTS Criterion 43). The test samples were fixed on an acrylic platform to facilitate experimental operations. 
The robotic manipulation is performed by achieving controllable vertical compression trajectories, with a speed range of 2 to 20 $mm/s$ and a maximum deformation depth of 1 to 20 millimeters, calibrated according to the specific mechanical properties of each material.  Data acquisition and real-time control were realized via
 a ROS-based computational architecture executing at 1$kHz$ on an Intel i7-9700K platform.

\begin{figure}[h]
\centering
\includegraphics[width=0.5\columnwidth, height=6cm, keepaspectratio]{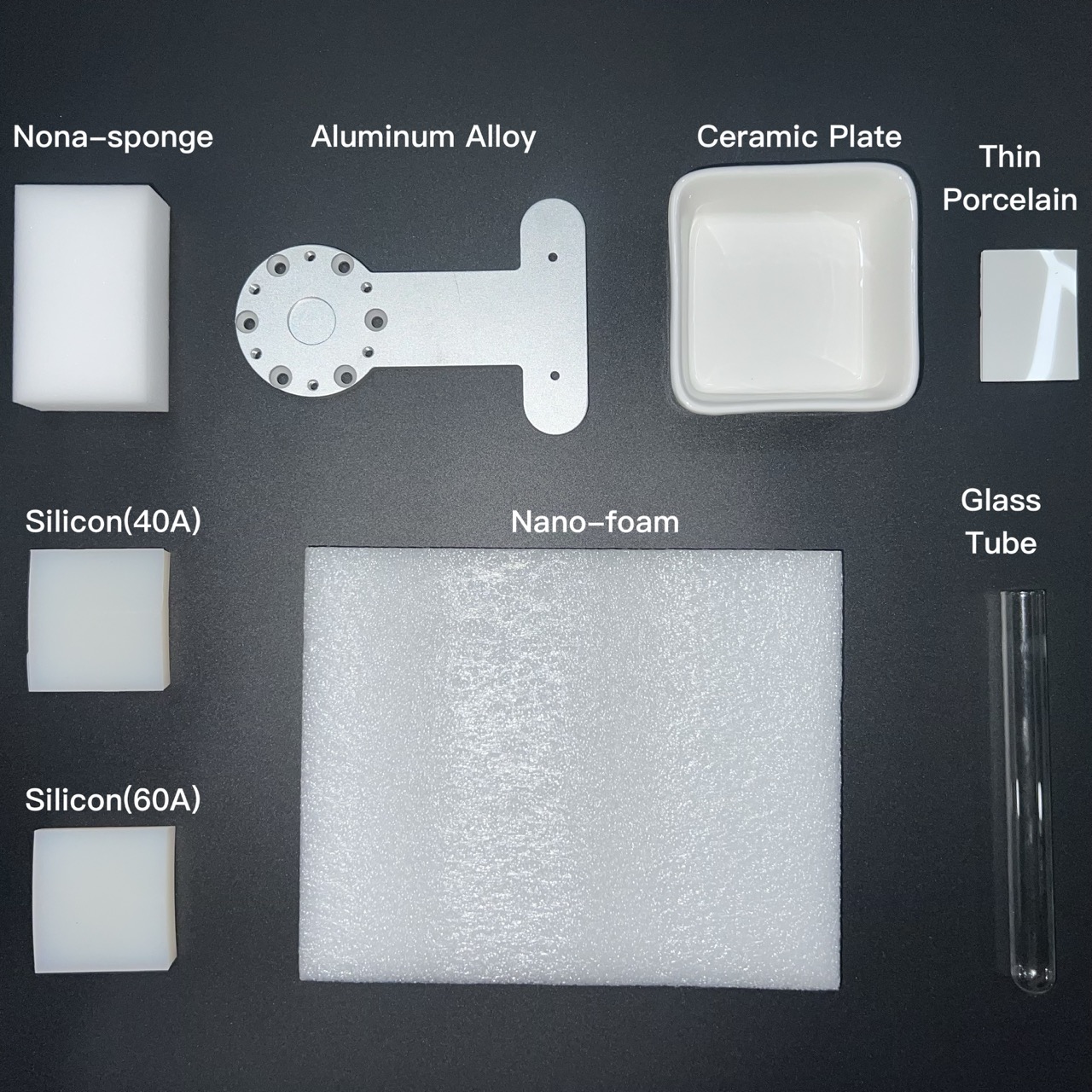}
\caption{Test items of viscoelastic objects.}
\label{fig:obj_press}
\end{figure}

\begin{table}[h]
\centering
\caption{Material Properties for Experimental Validation}
\label{tab:material_properties}
\renewcommand{\arraystretch}{1.2}
\setlength{\tabcolsep}{5pt}
\begin{tabular}{|p{2.2cm}|p{1.8cm}|p{2.8cm}|}
\hline
\textbf{Material Class} & \textbf{Variant} & \textbf{Properties} \\
\hline\hline
\multirow{2}{*}{Industrial Silicone} 
& Shore 40A & $E=1.8\pm0.2$ MPa, $\nu=0.48$ \\
\cline{2-3}
& Shore 60A & $E=3.5\pm0.3$ MPa, $\nu=0.47$ \\
\hline
\multirow{2}{*}{Precision Metals} 
& Aluminum & HB150, $E=69.5$ GPa \\
\cline{2-3}
& Steel & HB280, $E=210$ GPa \\
\hline
Nano-foam & $\rho=30$ kg/m$^3$ & $E=0.5\pm0.1$ MPa, $\nu=0.20$ \\
\hline
\multirow{2}{*}{Fragile Items} 
& Ceramic & 1 mm thickness \\
\cline{2-3}
& Glass Tubes & $\emptyset 8$ mm, 0.5 mm wall \\
\hline
\end{tabular}
\end{table}

\begin{figure}[htbp]
\centering  
\subfigure[]{   
\begin{minipage}{3cm}
\centering    
\includegraphics[scale=0.02]{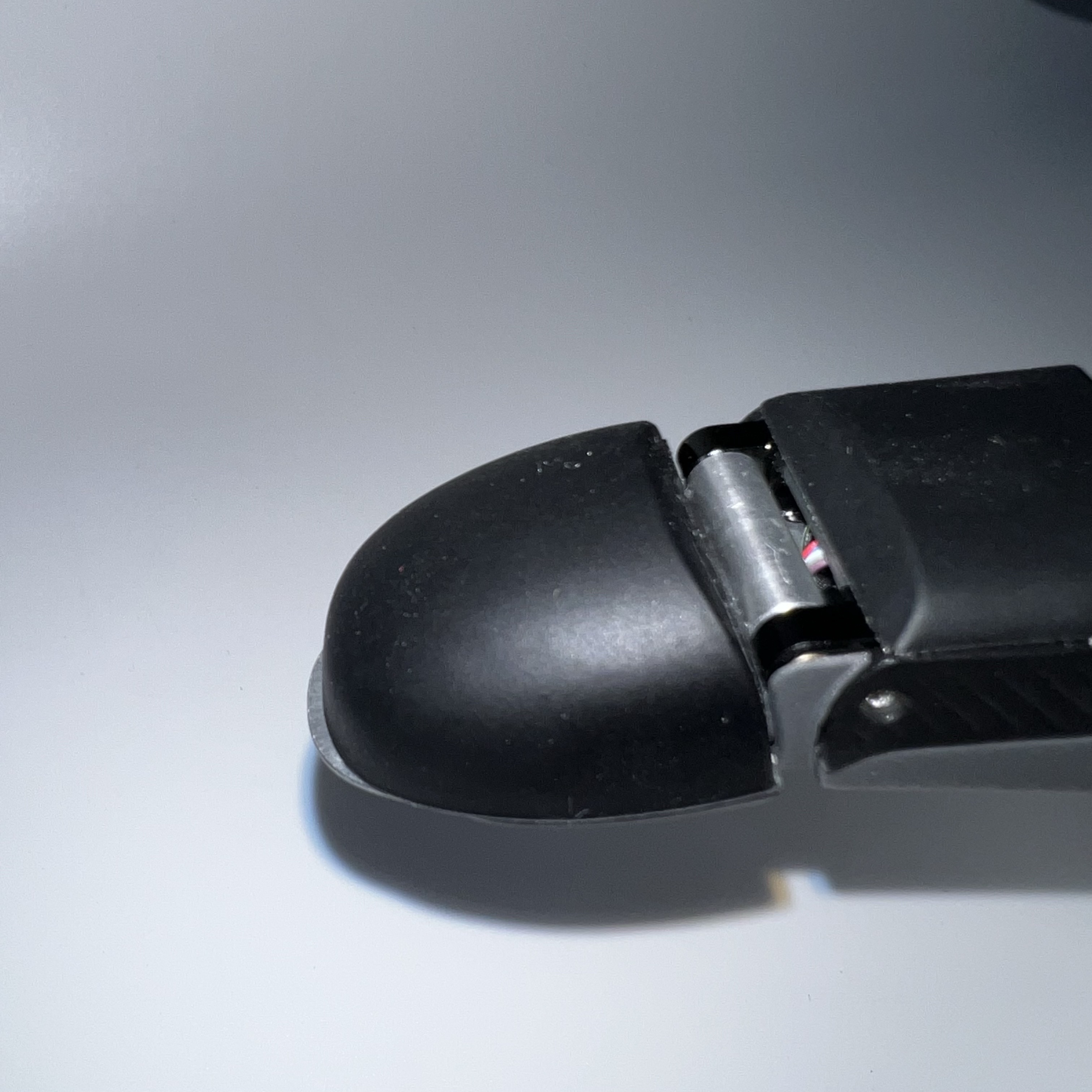}  
        \label{fig:tactile_physical}
\end{minipage}
}
\subfigure[]{ 
\begin{minipage}{3cm}
\centering    
\includegraphics[scale=0.042]{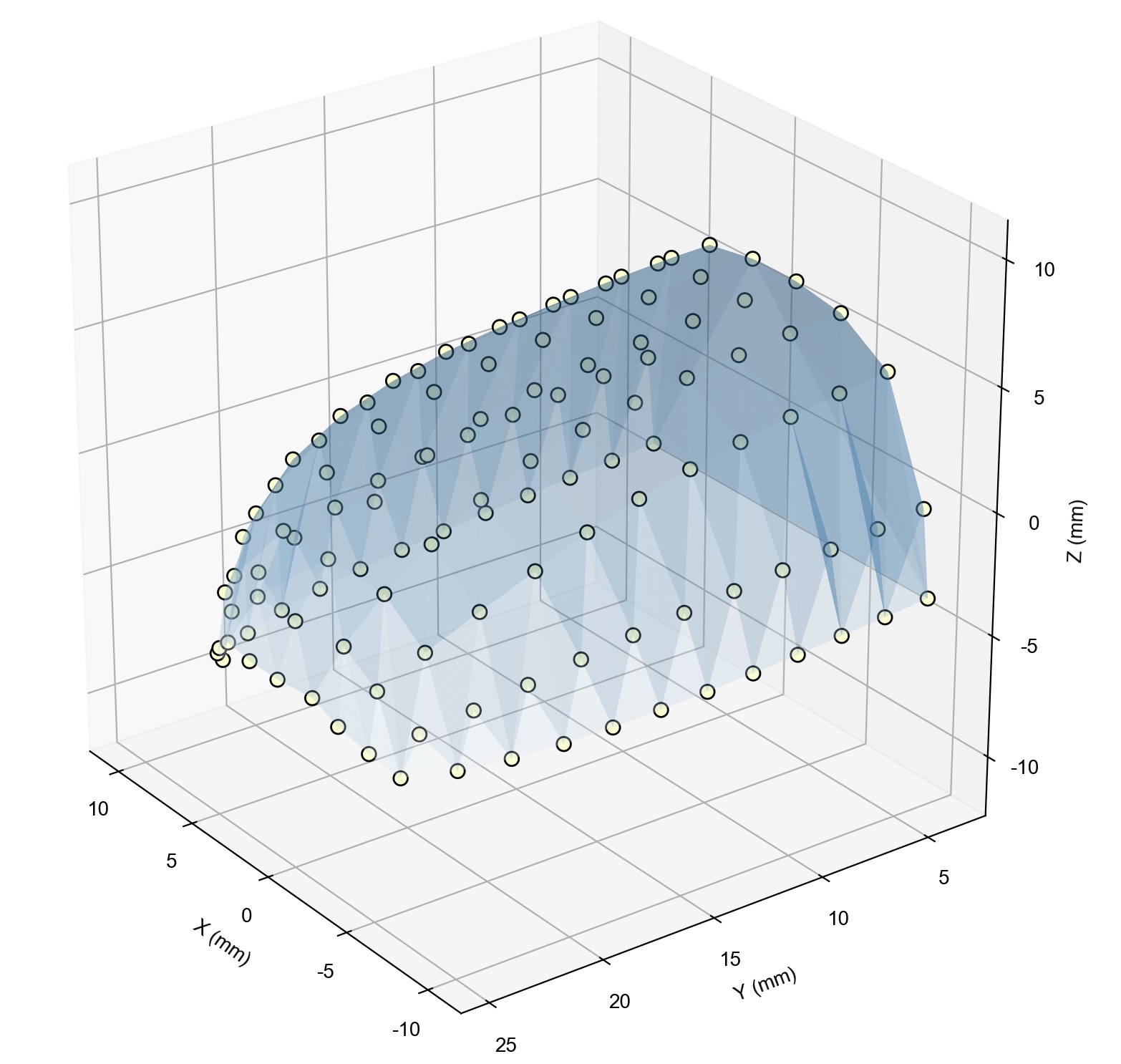}
        \label{fig:tactile_simulation}
\end{minipage}
}
\caption{PaXini fingertips DexH13 with high-resolution tactile sensor arrays PX6AX-GEN2-DP-M2826: ~(a) physical diagram; ~~(b) rendering diagram.}    
    \label{fig:tactile_sensor_configuration}
\end{figure}

 \begin{figure*}[h]
    \centering
    \includegraphics[width=0.9\textwidth]{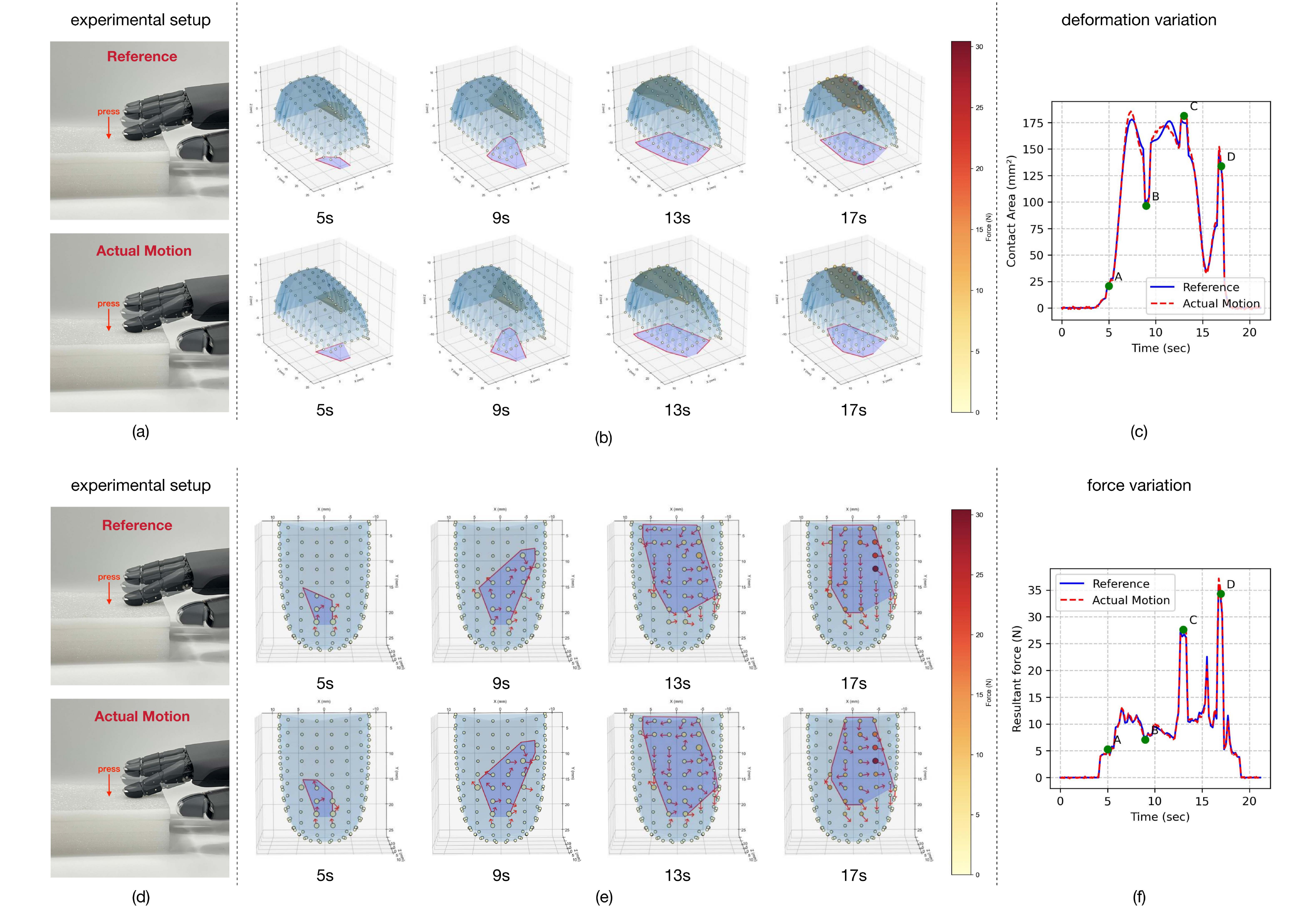}
    \caption{Foam compression test. Left: experimental setup and process. Middle: force-deformation response. Right: deformation and resultant force curves.}
    \label{fig:deformatio_and_force}
\end{figure*}

 \begin{table*}[h]
\centering
\caption{Force Control and Deformation Performance at Different Time Points}
\label{tab:performance_comparison}
\renewcommand{\arraystretch}{1.2}
\setlength{\tabcolsep}{5pt}
\begin{tabular}{|p{2.2cm}|c|c|c|c|c|c|c|c|}
\hline
\multirow{2}{*}{\textbf{Material}} & \multicolumn{4}{c|}{\textbf{Force Control Error (N)}} & \multicolumn{4}{c|}{\textbf{Deformation Error}} \\
\cline{2-9}
& \textbf{5s} & \textbf{9s} & \textbf{13s} & \textbf{17s} & \textbf{5s} & \textbf{9s} & \textbf{13s} & \textbf{17s} \\
\hline\hline
Silicone (40A) & $0.48 \pm 0.10$ & $0.46 \pm 0.10$ & $0.45 \pm 0.10$ & $0.45 \pm 0.10$ & $0.84 \pm 0.15$ & $0.82 \pm 0.15$ & $0.82 \pm 0.15$ & $0.82 \pm 0.15$ \\
\hline
Silicone (60A) & $0.58 \pm 0.13$ & $0.55 \pm 0.12$ & $0.55 \pm 0.12$ & $0.55 \pm 0.12$ & $0.80 \pm 0.14$ & $0.78 \pm 0.14$ & $0.78 \pm 0.14$ & $0.78 \pm 0.14$ \\
\hline
Nano-foam & $0.43 \pm 0.08$ & $0.38 \pm 0.07$ & $0.35 \pm 0.07$ & $0.35 \pm 0.07$ & $0.88 \pm 0.16$ & $0.86 \pm 0.16$ & $0.85 \pm 0.16$ & $0.85 \pm 0.16$ \\
\hline
Ceramic Plate & $0.23 \pm 0.05$ & $0.20 \pm 0.04$ & $0.20 \pm 0.04$ & $0.20 \pm 0.04$ & $0.73 \pm 0.13$ & $0.71 \pm 0.13$ & $0.71 \pm 0.13$ & $0.71 \pm 0.13$ \\
\hline
Glass Tube & $0.18 \pm 0.03$ & $0.15 \pm 0.03$ & $0.15 \pm 0.03$ & $0.15 \pm 0.03$ & $0.70 \pm 0.12$ & $0.68 \pm 0.12$ & $0.68 \pm 0.12$ & $0.68 \pm 0.12$ \\
\hline
Thin Porcelain & $0.15 \pm 0.03$ & $0.13 \pm 0.02$ & $0.13 \pm 0.02$ & $0.13 \pm 0.02$ & $0.67 \pm 0.11$ & $0.65 \pm 0.11$ & $0.65 \pm 0.11$ & $0.65 \pm 0.11$ \\
\hline
\multicolumn{9}{l}{\small Note 1: Force control errors are calculated based on the resultant force magnitude} \\
\multicolumn{9}{l}{\small Note 2: Deformation tracking errors are calculated using the dual-validation index.} \\
\end{tabular}
\end{table*}

\subsection{Force Tracking Performance}

The force control performance for various materials is verified
by applying a dynamic target force field (the upper row of Fig.\ref{fig:deformatio_and_force}(e))
and monitoring force distribution over 17 seconds, sampled at $t=5s, 9s, 13s, 17s$. 
For comparison, the force field on nano-foam is shown in the lower row of  Fig.\ref{fig:deformatio_and_force}(e),
revealing an evolution from initial contact at $t=5s$ to stable force maintenance at $t=17s$.
 The responses of tracking errors in the resultant forces of all materials demonstrate excellent steady-state performance  
in the left 4 rows in Tab. II, and the deformation tracking errors as the right 4 rows.
For all materials, the controller achieved precise force maintenance, with the total force tracking errors gradually decreasing from the initial contact to a stable state. Specifically, nano-foam samples showed a gradual control accuracy, 
with errors reducing from \(0.43 \pm 0.08 N\)  at $t=5s$  to \(0.35 \pm 0.07 N\) at steady state. Similar stability was observed in silicone samples, with 40A and 60A variants maintaining consistent performance after initial settling (\(0.45 \pm 0.10 N\) and \(0.55 \pm 0.12 N\), respectively). Rigid materials achieved even lower force tracking errors: $0.20\pm0.04N$ for ceramic plates and $0.15\pm0.03N$ for glass tubes. In this test, force control errors remained below $5\%$ across all materials.



\subsection{Deformation Control Precision}



The tactile sensor array generates continuous 3D surface topography data (Fig.\ref{fig:deformatio_and_force}(b)), enabling precise quantification of contact morphology evolution. Each tactile unit's activation represents a spatial deformation event, with corresponding coordinates and force magnitude, forming a deformation field that captures the material's response under force. Synchronized acquisition of deformation profiles (Fig.\ref{fig:deformatio_and_force}(c)) and resultant forces (Fig.\ref{fig:deformatio_and_force}(f)) allows comprehensive characterization of viscoelastic responses across material domains. A dual-validation index is defined to assess deformation, which combines spatial distribution accuracy and absolute area size by a weighted sum, i.e., $\epsilon_{total} = \alpha\cdot\epsilon_{dist} + \beta\cdot\epsilon_{area}$, where $\epsilon_{dist} $ is the quadratic mean of $(p_i^r - p_i^a)$ with
$p_i^r$, $p_i^a$  the coordinates of reference points and actual activated tactile points  respectively in the taxel field visualization (Fig.\ref{fig:deformatio_and_force}(e)), $\epsilon_{area} = |A_{target} - A_{actual}|$ with  $A_{target}$ and $A_{actual}$ the referenced areas and the measured contact areas respectively,  by performing convex hull calculation of the activated taxel distribution as visualized in Fig.\ref{fig:deformatio_and_force}(c). $\alpha=0.4$ and $\beta=0.6$. 
 The experimental results, as shown in Tab.\ref{tab:performance_comparison}, demonstrate consistent deformation control precision across different materials and time points.  For viscoelastic materials such as nano-foam, the deformation tracking error shows gradual accuracy from $0.88\pm0.16$ at $t=5s$ to $0.85\pm0.16$ at $t=17s$.  Similar stability is observed in silicone specimens, with 40A and 60A variants achieving stable deformation control with errors of $0.82\pm0.15$ and $0.78\pm0.14$ respectively.
Rigid materials exhibited even more precise deformation control, with ceramic plates maintaining deformation errors of $0.71\pm0.13$ and glass tubes achieving $0.68\pm0.12$ at steady state. These results demonstrate the controller's ability to maintain precise deformation while simultaneously ensuring stable force application across diverse material properties.
Fig.\ref{fig:deformatio_and_force}(c) illustrates the temporal evolution of contact area during a representative foam compression sequence. The deformation field (Fig.\ref{fig:deformatio_and_force}(b)) shows the progression of contact states at 
$t=5s, 9s, 13s, 17s$, corresponding to the key transition points (A, B, C, and D) marked in the deformation curve. 
The controller exhibits precise tracking of complex deformation profiles, maintaining synchronization with reference trajectories even during rapid transitions between these temporal states.  During the entire operation process, the gradual changes in contact geometry and force distribution clearly demonstrate the controller's ability to maintain a stable contact pattern.
The statistical analysis of tracking accuracy indicates that, under all experimental conditions, the standard deviations of the temporal error distributions are less than $5.0\%$ of the target values.

\section{CONCLUSIONS}

This paper studies the precise manipulation of viscoelastic material by developing a framework of CATCH-FORM-3D, combining 3D Kelvin–Voigt and Maxwell dynamics into a unified PDE model, and a dual-loop architecture integrating physics-induction admittance
control with a reaction- diffusion PDE-based boundary control. Experiments on several materials achieve sub-millimeter deformation accuracy and $\pm 5 \%$ force deviation, demonstrating robust precision in dynamic environments.  The proposed parameter estimator, admittance regulator and boundary controller are suitable for promotion to a wider range of application scenarios, such as
industrial manufacturing, polymer molding, medical procedures, and everyday domestic activities.

\addtolength{\textheight}{-12cm}   





\section*{ACKNOWLEDGMENT}


Authors are especially grateful to PaXini Technology for assisting in the construction of the experimental  platform.


\end{document}